\documentclass[twoside]{article}

\usepackage[accepted]{aistats2019}
\usepackage[round]{natbib}

\usepackage{hyperref}
\usepackage{subfig}
\usepackage{graphicx}
\usepackage{color}
\usepackage{tikz}
\usetikzlibrary{shapes,decorations,arrows,calc,arrows.meta,fit,positioning}
\tikzset{
    -Latex,auto,node distance =1 cm and 1 cm,semithick,
    state/.style ={ellipse, draw, minimum width = 0.7 cm},
    point/.style = {circle, draw, inner sep=0.04cm,fill,node contents={}},
    square/.style = {rectangle, draw, inner sep=0.04cm,fill,node contents={}},
    bidirected/.style={Latex-Latex,dashed},
    el/.style = {inner sep=2pt, align=left, sloped}
}
\usepackage{amsmath}
\usepackage{bm}
\usepackage{amssymb}
\usepackage{amsthm}
\usepackage[ruled,vlined]{algorithm2e}

\newcommand{\ci}{\perp\!\!\!\perp}
\newcommand{\intersect}{\cap}

\newtheorem{theorem}{Theorem}
\newtheorem{definition}{Definition}
\newtheorem{lemma}{Lemma}
\newtheorem{corollary}{Corollary}
\newtheorem{prop}{Proposition}
\newtheorem{innercustomthm}{Theorem}
\newenvironment{customthm}[1]
  {\renewcommand\theinnercustomthm{#1}\innercustomthm}
  {\endinnercustomthm}
\newtheorem{innercustomlemma}{Lemma}
\newenvironment{customlemma}[1]
  {\renewcommand\theinnercustomlemma{#1}\innercustomlemma}
  {\endinnercustomlemma}
\DeclareMathOperator*{\argmin}{\arg\!\min}

\begin{document}

\runningtitle{Learning Predictive Models That Transport}


\twocolumn[

\aistatstitle{Preventing Failures Due to Dataset Shift:\\ Learning Predictive Models That Transport}

%\aistatsauthor{ Author 1 \And Author 2 \And  Author 3 }
\aistatsauthor{ Adarsh Subbaswamy \And Peter Schulam \And  Suchi Saria }

%\aistatsaddress{ Institution 1 \And  Institution 2 \And Institution 3 }
\aistatsaddress{ Johns Hopkins University \And  Johns Hopkins University \And Johns Hopkins University } ]

\begin{abstract}
Classical supervised learning produces unreliable models when training and target distributions differ, with most existing solutions requiring samples from the target domain. We propose a proactive approach which learns a relationship in the training domain that will generalize to the target domain by incorporating prior knowledge of aspects of the data generating process that are expected to differ as expressed in a causal selection diagram. Specifically, we remove variables generated by unstable mechanisms from the joint factorization to yield the Surgery Estimator---an interventional distribution that is invariant to the differences across environments. We prove that the surgery estimator finds stable relationships in strictly more scenarios than previous approaches which only consider conditional relationships, and demonstrate this in simulated experiments. We also evaluate on real world data for which the true causal diagram is unknown, performing competitively against entirely data-driven approaches.
\end{abstract}

\section{INTRODUCTION}
As machine learning systems are increasingly deployed in practice, system developers are being faced with deployment environments that systematically differ from the training environment. However, models are typically evaluated by splitting a single dataset into train and test subsets such that training and evaluation data are, by default, drawn from the same distribution. When evaluated beyond this initial dataset, say in the deployment environment, model performance may significantly deteriorate and potentially cause harm in safety-critical applications such as healthcare (see e.g., \citet{schulam2017reliable,zech2018variable}). Because access to deployment environment data may not be available during training, it is not always feasible to employ \emph{domain adaptation} techniques to directly optimize the model for the target domain. This motivates the need for \emph{proactive} approaches which anticipate and address the differences between training and deployment environments without using deployment data \citep{subbaswamycounterfactual}. As a step towards building more reliable systems, in this paper we address the problem of proactively training models that are robust to expected changes in environment. 

In order to ensure reliability, we must first be able to identify the sources of the changes. One way to do this is to reason about the differences in the underlying data generating processes (DGPs) that produce the data. For example, suppose we wish to diagnose a target condition $T$, say lung cancer, using information about patient chest pain symptoms $C$ and whether or not they take aspirin $A$. From our prior knowledge of the DGP we know that lung cancer leads to chest pain and that aspirin can relieve chest pain. We also know that smoking $K$ (unrecorded) is a risk factor for both lung cancer and heart disease, and aspirin is prescribed to smokers as a result. A diagnostic tool for this problem will be trained from one dataset before being deployed in hospitals that may not be represented in the data. Still, a modeler can reason about which aspects of the DGP are likely to differ across hospitals. For example, while the effects of lung cancer or aspirin on chest pain will not vary across hospitals, the policy used to prescribe aspirin to smokers (i.e., $P(A|K)$) is practice dependent and will vary.

What can the modeler do after identifying potential sources of unreliability in the data? Because the modeler does not know which prescription policies will be in place at deployment locations or by how much the deployment DGP will differ from the training DGP, the modeler should design the system to be \emph{stable} (i.e., invariant) to the differences in prescription policy. This means the model should predict using only the pieces of the DGP that are expected to stay the same across environments while not learning relationships that make use of the varying parts of the DGP. If the model predictions somehow depend on the prescription policy, then when the deployment policy strongly differs from the training policy model performance will significantly degrade and aspirin-taking subpopulations will be systematically misclassified.

To ensure that a model does not make use of unreliable relationships in its predictions (i.e., relationships involving prescription policy), it helps to have a representation of the DGP that makes explicit our assumptions about the DGP and what we expect will vary. A natural representation is to use \emph{selection diagrams} \citep{pearl2011transportability}, which consist of two types of nodes: nodes representing variables relevant to the DGP (e.g., smoking $K$ or lung cancer $T$) and auxiliary \emph{selection variables} (denoted by square $S$ nodes) which identify sources of unreliability in the DGP. For example, the selection diagram in Figure \ref{fig:diagnosis}a represents the DGP underlying the diagnosis example, with the selection variable $S$ pointing to the piece of the DGP we expect to vary: the aspirin prescription policy $P(A|K)$. The selection variables point to \emph{mutable} variables that are generated by mechanisms that are expected to differ across environments such that a selection diagram represents a family of DGPs which differ only in the mechanisms that generate the mutable variables.

%% [FIGURE 1; DAG FOR EXAMPLE]
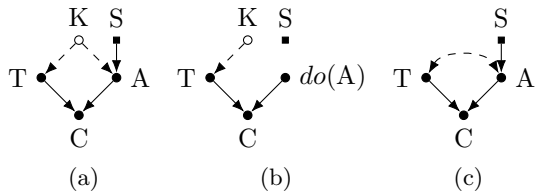
\begin{figure}[!t]
\vspace{-0.2in}
\centering
     \subfloat[]{%
      \begin{tikzpicture}
          \def\unit{0.5}
            \node[fill=black!0] (k) at (0,0) [label=above:K,point];
            \node (a) at (\unit, -\unit) [label=right:A,point];
            \node (t) at (-\unit,  -\unit) [label=left:T,point];
            \node (c) at (0, -2*\unit) [label=below:C,point];
            \node (s) at (\unit, 0) [label=above:S,square];
            % Directed edge
            \path[dashed] (k) edge (t);
            \path[dashed] (k) edge (a);
            \path (t) edge (c);
            \path (a) edge (c);
            \path (s) edge (a);
        \end{tikzpicture}
     }
     \subfloat[]{%
      \begin{tikzpicture}
            \def\unit{0.5}
            \node[fill=black!0] (k) at (0,0) [label=above:K,point];
            \node (a) at (\unit, -\unit) [label=right:$do$(A),point];
            \node (t) at (-\unit,  -\unit) [label=left:T,point];
            \node (c) at (0, -2*\unit) [label=below:C,point];
            \node (s) at (\unit, 0) [label=above:S,square];
            % Directed edge
            \path[dashed] (k) edge (t);
            % \path[dashed] (s) edge (a);
            \path (t) edge (c);
            \path (a) edge (c);
        \end{tikzpicture}
     }
      \subfloat[]{%
      \begin{tikzpicture}
          \def\unit{0.5}
            \node (a) at (\unit, -\unit) [label=right:A,point];
            \node (t) at (-\unit,  -\unit) [label=left:T,point];
            \node (c) at (0, -2*\unit) [label=below:C,point];
            \node (s) at (\unit, 0) [label=above:S,square];
            % Directed edge
            \path[bidirected] (t) edge[bend left = 60] (a);
            \path (t) edge (c);
            \path (a) edge (c);
            \path (s) edge (a);
        \end{tikzpicture}
     }
\caption{(a) Selection diagram for the diagnosis example. $K$ is unobserved. (b) DAG after performing graph surgery. (c) The ADMG yielded by taking the latent projection of (a).}
\label{fig:diagnosis}
\vspace{-0.2in}
\end{figure}

Checking model stability graphically using selection diagrams is straightforward: a model is stable if its predictions are independent of the selection variables. If predictions are not independent of the selection variables then they depend on the environment-varying mechanisms that do not generalize. To illustrate, suppose smoking status $K$ is not recorded in the data (denoted by the dashed edges in Figure \ref{fig:diagnosis}a). A discriminative model of $P(T|A,C)$ which conditions on all recorded features will be dependent on the selection variable: $P(T|A,C) \not = P(T|A,C,S)$, which indicates this distribution is unstable (i.e., it differs by environment). One solution (see e.g., \citet{subbaswamycounterfactual,magliacane2018causal}), which we term \emph{graph pruning}, is to perform \emph{feature selection} to find a stable subset of features $\mathbf{Z}$ such that the conditional distribution transfers: $P(T|\mathbf{Z}) = P(T|\mathbf{Z}, S)$. However, for the problem in Figure \ref{fig:diagnosis}a, the only stable set is $\mathbf{Z} = \emptyset$ because by $d$-separation \citep{koller2009probabilistic} conditioning on either $A$ or $C$ activates the path $T \leftarrow K \rightarrow A \leftarrow S$, inducing a dependence between $T$ and $S$. Further, in cases where the mechanism that generates $T$ varies across environments (i.e., the \emph{target shift} scenario in which $S$ is a parent of $T$), no stable feature set exists. While without more assumptions or external data there is no stable discriminative model for such problems, in this paper we relax these limitations and propose an approach which can recover stable predictive relationships in cases where graph pruning fails. 

The proposed solution, which we term the \emph{Graph Surgery} estimator,\footnote{Henceforth graph surgery, surgery estimator, or surgery.} is to directly remove any possible dependence on environment-varying mechanisms by using \emph{interventional} \citep{pearl2009causality} rather than observational distributions to predict. Specifically, we consider a hypothetical intervention in which for each individual the mutable variables are set to the values they were actually observed to be (i.e., $do(A)$ in our example). Graphically, the intervention $do(A)$ results in a mutilated graph (Figure \ref{fig:diagnosis}b) in which the edges into $A$, including edges from selection variables, are ``surgically'' removed \citep{pearl1998graphical}. The resulting interventional distribution\footnote{We will use $p(Y|do(X))$ and $P_{X}(Y)$ interchangeably.}
\begin{align*}
    P_A(T|C) = P(T|C,do(A)) &\propto P(T,C|do(A)) \\
    &= P(T) P(C|T,A)
\end{align*}
is invariant to changes in how $A$ is generated, reflecting the ``independence of cause and mechanism'' \citep{peters2017elements}, ensuring stability, and allowing us to use information about $A$ and $C$ that the graph pruning solution ($P(T)$) does not.  Graph surgery can be seen as learning a predictor from an alternate, hypothetical DGP in which the mutable variables were generated by direct assignment rather than by the environment-specific mechanisms. This severs dependence on selection variables to yield a stable predictor. One challenge is that when the DAG contains hidden variables (as is common in reality), interventional distributions are not always uniquely expressible as a function of the observational training data \citep{pearl2009causality}. To address this we use the previously derived ID algorithm \citep{tian2002general,shpitser2006identification} for determining identifiability of interventional distributions.

\textbf{Contributions:} We propose the graph surgery estimator, an algorithm for estimating stable predictive models that can generalize even when train and test distributions differ. Graph surgery depends on a causal DAG to encode prior information about how the distribution of data might change. Given this prior information, it produces a predictor that does not depend on these unreliable parts of the data generating process. We show that graph surgery relaxes limiting assumptions made by existing methods for learning stable predictors. In addition, we connect the optimality of graph surgery to recently proposed adversarial distributional robustness problems.

\section{RELATED WORK}
Differences between training and test distributions have been previously studied as the problem of \emph{dataset shift} \citep{candela2009dataset}. Many specific forms of dataset shift have been characterized by dividing the variables into the input features and the target prediction outcome. By reasoning about the causal relationship between the inputs and target, various forms of dataset shift can be defined \citep{storkey2009training,scholkopf2012causal} which has led to methods for tackling specific instances such as \emph{covariate shift} (e.g., \citet{sugiyama2007covariate,gretton2009}), \emph{target shift} \citep{zhang2013domain,lipton2018detecting}, \emph{conditional shift} \citep{zhang2015multi, gong2016domain}, and \emph{policy shift} \citep{schulam2017reliable}. Using selection diagrams we can consider complex dataset shift scenarios beyond these two variable-type settings.

One issue is that methods for addressing dataset shift have mainly been \emph{reactive}: they make use of unlabeled data from the target domain to reweight training data during learning and optimize the model specifically for the target domain (e.g., \citet{storkey2009training,gretton2009}). However, if we do not have target domain data from every possible environment during learning, we must instead use \emph{proactive} approaches in which the target domain remains unspecified \citep{subbaswamycounterfactual,saria2019tutorial}.

One class of proactive solutions considers bounded \emph{distributional robustness}. These methods assume that the possible test distributions are in some way centered around the training distribution. For example, in adversarial learning \citet{sinha2018certifying} consider a Wasserstein ball around the training distribution. \citet{rothenhausler2018anchor} assume that differences between train and test distributions are bounded magnitude shift perturbations. However, these methods fail to give robustness guarantees on perturbations that are beyond the prespecified magnitude used during training. In safety-critical applications where preventing failures is crucial, we require unbounded invariance to perturbations which motivates the use of causal-based methods \citep{meinshausen2018causality}.

To achieve stable models with complete invariance to perturbations, \emph{graph pruning} methods consider a feature selection problem in which the goal is to find the optimal subset that makes the target independent from the selection variables. \citet{rojas2018invariant} and \citet{magliacane2018causal} accomplish this by empirically determining a stable conditioning set by hypothesis testing the stability of the set across multiple source domains and assuming that the target variable is not generated by a varying mechanism (no $S \rightarrow T$ edge). Extending this, \citet{subbaswamycounterfactual} consider also adding \emph{counterfactual} variables to stable conditioning sets which allow the model to make use of more stable information than by using observed variables alone. However, this requires the strong parametric assumption that causal mechanisms are linear. By using interventional distributions rather than counterfactuals, graph surgery is able to relax this assumption and nonparametrically use more stable information than observational conditional distributions. Additionally, graph surgery allows for the target to be generated by a varying mechanism.

\section{METHODS}
Our goal is to find a predictive distribution that generalizes even when train and test distributions differ. Derivation of the surgery estimator requires explicitly reasoning about the aspects of the DGP that can change and results in an interventional distribution in which the corresponding terms have been deleted from the factorization of the training distribution. In Section \ref{subsec:prelims} we introduce requisite prior work on identifying interventional distributions before presenting the surgery estimator in Section \ref{subsec:surgery} and establishing its soundness and completeness in Section \ref{subsec:sound}.

\subsection{Preliminaries}\label{subsec:prelims}
\textbf{Notation:} Throughout the paper sets of variables are denoted by bold capital letters while their particular assignments are denoted by bold lowercase letters. We will consider graphs with directed or bidirected edges (e.g., $\leftrightarrow$). Acyclic will be taken to mean that there exists no purely directed cycle. The sets of parents, children, ancestors, and descendants in a graph $\mathcal{G}$ will be denoted by $pa_\mathcal{G}(\cdot)$, $ch_\mathcal{G}(\cdot)$, $an_\mathcal{G}(\cdot)$, and $de_\mathcal{G}(\cdot)$, respectively. Our focus will be causal DAGs whose nodes can be partitioned into sets $\mathbf{O}$ of observed variables, $\mathbf{U}$ of unobserved variables, and $\mathbf{S}$ of selection variables. $\mathbf{O}$ and $\mathbf{U}$ consist of variables in the DGP, while $\mathbf{S}$ are auxiliary variables that denote mechanisms of the DGP that vary across environments.

\textbf{Interventional Distributions:} We now build up to the Identification (ID) algorithm
\citep{tian2002general,shpitser2006identification}, a sound and complete algorithm  \citep{shpitser2006identification} for determining whether or not an interventional distribution is identifiable, and if so, its form as a function of observational distributions. The ID algorithm operates on a special class of graphs known as \emph{acyclic directed mixed graph} (ADMGs). Any hidden variable DAG can be converted to an ADMG by taking its \emph{latent projection} onto $\mathbf{O}$ \citep{verma1991equivalence}. In the latent projection $\mathcal{G}'$ of a DAG $\mathcal{G}$ over observed variables $\mathbf{O}$, for $O_i, O_j \in \mathbf{O}$ there is an edge $O_i \rightarrow O_j$ if there exists a directed path from $O_i$ to $O_j$ in $\mathcal{G}$ where all internal nodes are unobserved, and $O_i \leftrightarrow O_j$ if there exists a divergent path from $O_i$ to $O_j$ (e.g., $O_i \leftarrow U \rightarrow O_j$) in $\mathcal{G}$ such that all internal nodes are unobserved. The bidirected edges represent \emph{unobserved confounding}. Figure \ref{fig:diagnosis}c shows the latent projection of the DAG in Figure \ref{fig:diagnosis}a. The joint distribution of an ADMG factorizes as:
\begin{equation}\label{eq:admg}
    P(\mathbf{O}) = \sum_{\mathbf{U}} \prod_{O_i \in \mathbf{O}} P(O_i|pa(O_i)) P(\mathbf{U}).
\end{equation}
% ADMGs also factorize over their \emph{c-components}, and this property is used in the ID Algorithm.
% \begin{definition}[C-component]
% In an ADMG, a c-component consists of a maximal subset of observed variables that are connected to each other through bidirected paths. A vertex with no incoming bidirected edges forms its own c-component.
% \end{definition}

An intervention on $\mathbf{X}=\mathbf{O}\setminus \mathbf{V}$ sets these variables to constants $do(\mathbf{x})$. As constants, $P(x|do(x))=1$ such that $\prod_{X_i\in\mathbf{X}}P(X_i|pa(X_i))$ are deleted from (\ref{eq:admg}) to yield the interventional distribution:
\begin{equation*}
    P_{\mathbf{X}}(\mathbf{V}) = \sum_\mathbf{U} \prod_{V_i \in \mathbf{V}} P(V_i|pa(V_i)) P(\mathbf{U}).
\end{equation*}
Graphically, the intervention results in the mutilated graph $\mathcal{G}_{\mathbf{\overline{X}}}$ in which the edges into  $\mathbf{X}$ have been removed.\footnote{Similarly, $\mathcal{G}_{\mathbf{\underline{X}}}$ will denote a mutilated graph in which edges out of $\mathbf{X}$ are removed.} When ADMG $\mathcal{G}$ contains bidirected edges, interventional distributions are not always \emph{identifiable}.
\begin{definition}[Causal Identifiability]
For disjoint variable sets $\mathbf{X},\mathbf{Y} \subseteq \mathbf{O}$, the effect of an intervention $do(\mathbf{x})$ on $\mathbf{Y}$ is said to be identifiable from $P$ in $\mathcal{G}$ if $P_{\mathbf{x}}(\mathbf{Y})$ is (uniquely) computable from $P(\mathbf{O})$ in any causal model which induces $\mathcal{G}$.
\end{definition}
The ID algorithm (a version of it is shown in Appendix A) determines if a particular interventional distribution is identified. Specifically, given disjoint variable sets $\mathbf{X},\mathbf{Y}\subseteq \mathbf{O}$ and an ADMG $\mathcal{G}$, a function call to ID$(
\mathbf{X},\mathbf{Y};\mathcal{G})$ returns an expression (in terms of $P(\mathbf{O})$) for $P_\mathbf{X}(\mathbf{Y})$ if it is identified, otherwise it throws a failure exception. The ID algorithm is nonparametric, so the terms in the expression it returns can be learned from training data with arbitrary black box approaches. 

\begin{algorithm}[!t]
 \SetKwInOut{Input}{input}\SetKwInOut{Output}{output}
 \Input{ADMG $\mathcal{G}$, disjoint variable sets $\mathbf{X,Y,Z\subset O}$}
 \Output{Unconditional query $\propto P_\mathbf{X}(\mathbf{Y}|\mathbf{Z})$.}
 $\mathbf{X}' = \mathbf{X}$; $\mathbf{Y}' = \mathbf{Y}$; $\mathbf{Z}' = \mathbf{Z}$\;
 \While{$\exists Z\in \mathbf{Z}$ s.t. $(\mathbf{Y} \ci Z|\mathbf{X},\mathbf{Z}\setminus\{Z\})_{\mathcal{G}_{\mathbf{\overline{X}},\underline{Z}}},$}{
    $\mathbf{X}' = \mathbf{X}' \cup Z$\;
    $\mathbf{Z}' = \mathbf{Z}' \setminus \{Z\}$\;
 }
 $\mathbf{Y}' = \mathbf{Y} \cup \mathbf{Z}'$\;
 \KwRet $\mathbf{X}', \mathbf{Y}'$ of unconditional query $P_{\mathbf{X}'}(\mathbf{Y}')$
  \caption{Unconditional Query: UQ($\mathbf{X}$, $\mathbf{Y}$, $\mathbf{Z}$; $\mathcal{G}$)}
  \label{alg:idc}
\end{algorithm}

In \citet{shpitser2006idc}, the ID algorithm was extended to answer \emph{conditional effect} queries of the form $P_\mathbf{X}(\mathbf{Y}|\mathbf{Z})$ for disjoint sets $\mathbf{X,Y,Z}\subset \mathbf{O}$ by showing that every conditional ID query can be reduced to an unconditional ID query using the procedure shown in Algorithm \ref{alg:idc}. This procedure finds the maximal subset of variables in the conditioning set $\mathbf{Z}$ to bring into the intervention set using Rule 2 of $do$-calculus (action/observation exchange) \cite[Chapter 3]{pearl2009causality}. The resulting conditional interventional distribution is then proportional to the joint distribution of $\mathbf{Y}$ and the remaining variables in the conditioning set. A call to ID can then determine the identifiability of the resulting unconditional query.

\textbf{Transportability:} \emph{Transportability} is a framework for the synthesis of experimental and observational data from multiple environments to answer a statistical or causal query in a prespecified target environment \citep{pearl2011transportability,bareinboim2012transportability}. In order to build safe and reliable models, we restrict our attention to learning predictive models that can be \emph{directly transported} from the source environment to an unspecified target environment without any adjustment.

\begin{definition}[Selection diagram]
A selection diagram is a causal DAG or ADMG augmented with auxiliary selection variables $\mathbf{S}$ (denoted by square nodes) such that for $S\in \mathbf{S}, X\in\mathbf{O}\cup\mathbf{U}$ an edge $S \rightarrow X$ denotes the causal mechanism that generates $X$ may vary arbitrarily in different environments. Selection variables may have at most one child.
\end{definition}
We refer to the children of $\mathbf{S}$ as \emph{mutable} variables. Selection diagrams define a family of distributions over environments such that $P(X|pa(X)),\forall X\in ch(\mathbf{S})$ in (\ref{eq:admg}) can differ arbitrarily in each environment. Constructing a selection diagram generally requires domain knowledge to specify the mechanisms and the placement of selection variables. Without prior knowledge \emph{causal discovery} methods can potentially be used \citep{spirtes2000causation}.

We now define \emph{stability} as a predictive analog of \emph{direct transportability} \citep{pearl2011transportability}, in which a source environment relationship holds in the target environment without adjustment.
% \begin{definition}[Stable estimator]
% An estimator for predicting a variable $T$ is said to be stable if it is expressible as an identifiable function of the observed data distribution and is not a function of any $S \in \mathbf{S}$.
% \end{definition}
\begin{definition}[Stable estimator]
An estimator for predicting a variable $T$ is said to be stable if it is independent of all $S \in \mathbf{S}$.
\end{definition}
Graph pruning and graph surgery can both produce stable estimators, but pruning estimators will always be observational conditional distributions while surgery estimators will be the identified form of an interventional distribution.

\subsection{The Graph Surgery Estimator}\label{subsec:surgery}
Graph surgery assumes the data modeler has constructed or been given a causal DAG of the DGP with target prediction variable $T$, observed variables $\mathbf{O}$, and unobserved variables $\mathbf{U}$ that has been augmented with selection variables $\mathbf{S}$ using prior knowledge about mechanisms that are expected to differ across environments (e.g., prescription policy). An overview of the procedure is as follows: The selection DAG is converted to a selection ADMG so it is compatible with the ID algorithm. Children of $\mathbf{S}$ in the selection ADMG form the set of mutable variables $\mathbf{M}$. The proposed algorithm then searches all possible interventional distributions (which intervene on $\mathbf{M}$) for the optimal (with respect to held-out source environment data) identifiable  distribution, which is normalized and returned as the surgery estimator. We now cover each step in detail.

Only observed variables can be intervened on, so to determine $\mathbf{M}$, we take the latent projection of the selection DAG $\mathcal{H}$ to turn it into an ADMG $\mathcal{G}$. If a selection variable $S\in\mathbf{S}$ has multiple children in $\mathcal{G}$, then $S$ should be split into multiple selection variables, one per child, with the new selection variables added to $\mathbf{S}$. Any disconnected variables in $\mathbf{S}$ can be removed. The mutable variables are then given by $\mathbf{M}=ch_\mathcal{G}(\mathbf{S}) \subseteq \mathbf{O}$. We now establish that intervening on (at least) $\mathbf{M}$ results in a stable estimator.
\begin{prop}
For $\mathbf{Y}\subseteq \mathbf{O}$, $\mathbf{X} \supseteq \mathbf{M}$ such that $\mathbf{Y}\intersect\mathbf{X}=\emptyset$, the interventional distribution $P_\mathbf{X}(\mathbf{Y})$ is stable.
\end{prop}
\begin{proof}
The intervention $do(\mathbf{X})$ results in the graph $\mathcal{G}_{\overline{\mathbf{X}}}$ in which all edges into $\mathbf{X}$ are removed. Since $\mathbf{X} \supseteq \mathbf{M}$ and $\mathbf{M}=ch(\mathbf{S})$, the intervention removes all edges out of $\mathbf{S}$. This means $\mathbf{S}$ is disconnected (and thus $d$-separated) from $\mathbf{Y}$ in $\mathcal{G}_{\overline{\mathbf{X}}}$ which gives us stability.
\end{proof}

What interventional distribution should we use to predict $T$? A natural idea is to use the full conditional interventional distribution $P_{\mathbf{M}}(T|\mathbf{O}\setminus (\mathbf{M} \cup \{T\}))$ which can be turned into a corresponding unconditional query (so we can call ID) using a call to Algorithm \ref{alg:idc}: UQ$(\mathbf{M}, T, \mathbf{O}\setminus (\mathbf{M} \cup \{T\}); \mathcal{G})$. However, this has two issues. First, if the target variable is mutable itself ($T\in\mathbf{M}$) then the conditional interventional distribution is ill-defined since the three variable sets must be disjoint. If $T$ is mutable, then we must intervene on it, graphically represented by deleting all edges into $T$. Variables related to $T$ through edges out of $T$ (e.g., children and their bidirected neighborhoods) can still be used to predict $T$. Thus, if $T\in \mathbf{M}$, we can generate an unconditional query of the form $P_\mathbf{X}(\mathbf{Y})$ from UQ$(\mathbf{M} \setminus \{T\}, T, \mathbf{O}\setminus \mathbf{M}; \mathcal{G}_{\overline{T}})$, noting that we are using the mutilated graph $\mathcal{G}_{\overline{T}}$. We must further modify the result to account for the fact that we are also intervening on $T$: $P_{\mathbf{X} \cup \{T\}}(\mathbf{Y}\setminus \{T\})$. Importantly, there is never a stable pruning estimator when $T\in\mathbf{M}$ which shows that graph surgery can provide stability in cases where existing pruning solutions cannot.

Second, the full conditional interventional distribution may not be identifiable. We propose an exhaustive search over possible conditioning sets: trying each $P_{\mathbf{M}}(T|\mathbf{Z})$ for $\mathbf{Z} \in \mathcal{P}(\mathbf{W})$,  $\mathbf{W} = \mathbf{O}\setminus (\mathbf{M} \cup \{T\})$ where $\mathcal{P}(\cdot)$ denotes the power set. In the interest of identifiability, even if $T\not\in\mathbf{M}$ we may want to consider intervening on $T$.\footnote{Deleting edges in a graph generally helps identifiability \citep{pearl2009causality}.} For example, in Figure \ref{fig:examples}(a), $P_X(T|Y)$ and $P_X(T)$ are not identifiable, but $P_{X,T}(Y)$ is. Thus, we should consider the unconditional query returned by Algorithm \ref{alg:idc} in both $\mathcal{G}$ and $\mathcal{G}_{\overline{T}}$ (with the modification of moving $T$ to the intervention set). The full procedure is given as Algorithm \ref{alg:surgery}. Note that it returns the estimator that performs the best on held out source-environment validation data with respect to some loss function $\ell$. If there is no identifiable interventional distribution the Algorithm throws a failure exception.

\begin{algorithm}[!t]
\SetKwBlock{Try}{try}{end}
\SetKwBlock{Catch}{catch}{end}
\SetKw{pass}{pass}
\SetKw{continue}{continue}
\SetKwProg{Fn}{Function}{:}{}
 \SetKwInOut{Input}{input}\SetKwInOut{Output}{output}
 \Input{ADMG $\mathcal{G}$, mutable variables $\mathbf{M}$, target $T$}
 \Output{Expression for the surgery estimator or \texttt{FAIL} if there is no stable estimator.}
 Let $S_{ID}=\emptyset$; Let $Loss=\emptyset$\;
 \For{$\mathbf{Z}\in\mathcal{P}(\mathbf{\mathbf{O}\setminus(\mathbf{M}\cup \{T\})})$}{
    \If{$T\not\in\mathbf{M}$}{
        Let $\mathbf{X},\mathbf{Y}=$ UQ$(\mathbf{M}, \{T\}, \mathbf{Z};\mathcal{G})$\;
        % \If{$\mathbf{Y}\intersect ({T}\cup ch(T))=\emptyset$}{\continue\;}
        \Try{
            $P=$ID$(\mathbf{X},\mathbf{Y};\mathcal{G})$\;
            $P_s= P/\sum_T P$\;
            Compute validation loss $\ell(P_s)$\;
            $S_{ID}$.append($P_s$); $Loss$.append($\ell(P_s)$)\;
        }\Catch{
            \pass\;
        }
    }
    Let $\mathbf{X},\mathbf{Y}=$ UQ$(\mathbf{M}, \{T\}, \mathbf{Z};\mathcal{G}_{\overline{T}})$\;
    $\mathbf{X} = \mathbf{X}\cup\{T\}$; $\mathbf{Y}= \mathbf{Y} \setminus \{T\}$\;
    \If{$\mathbf{Y}\intersect ({T}\cup ch(T))=\emptyset$}{\continue\;}
    \Try{
        $P=$ID$(\mathbf{X},\mathbf{Y};\mathcal{G})$\;
        $P_s= P/\sum_T P$\;
        Compute validation loss $\ell(P_s)$\;
        $S_{ID}$.append($P_s$); $Loss$.append($\ell(P_s)$)\;
    }\Catch{
        \continue\;
    }
 }
 \If{$S_{ID}=\emptyset$}{\KwRet \texttt{FAIL}\;}
 \KwRet $P_s\in S_{ID}$ with lowest corresponding $Loss$\;
 \caption{Graph Surgery Estimator}
 \label{alg:surgery}
\end{algorithm}

\subsection{Soundness and Completeness}\label{subsec:sound}
Algorithm \ref{alg:surgery} is sound in that it only returns stable estimators and complete in that it finds a stable surgery estimator if one exists. Proofs are in the supplement.

\begin{theorem}[Soundness]
When Algorithm \ref{alg:surgery} returns an estimator, the estimator is stable.
\end{theorem}
% \begin{proof}
% See supplement
% % Any query Algorithm \ref{alg:surgery} makes to ID considers intervening on a superset of the mutable variables $\mathbf{X} \supseteq \mathbf{M}$. By Proposition 1 this means the target interventional distribution is stable. From the soundness of the ID algorithm \cite[Theorem 5]{shpitser2006identification}, the resulting functional of observational distributions that Algorithm \ref{alg:surgery} returns will be stable.
% \end{proof}

\begin{theorem}[Completeness]\label{theorem:complete}
If Algorithm \ref{alg:surgery} fails, then there exists no stable surgery estimator for predicting $T$.
\end{theorem}
% \begin{proof}%[Proof Sketch]
% See Supplement
% % Algorithm \ref{alg:surgery} is an exhaustive search over interventional distributions that intervene on supersets of $\mathbf{M}$ and are functions of $T$. Thus, if there is a stable surgery estimator, the procedure will find one.
% \end{proof}

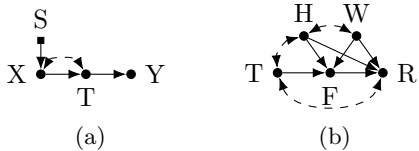
\begin{figure}[!t]
\vspace{-0.2in}
\centering
    \subfloat[]{%
      \begin{tikzpicture}
            \def\unit{0.3}
            \node (x) at (0,0) [label=left:X,point];
            \node (t) at (2*\unit, 0) [label=below:T,point];
            \node (s) at (0, 1.5*\unit) [label=above:S, square];
            \node (y) at (4*\unit, 0) [label=right:Y, point];
            
            \path (s) edge (x);
            \path (x) edge (t);
            \path (t) edge (y);
            \path[bidirected] (x) edge[bend left=60] (t);
        \end{tikzpicture}
     }\qquad
     \subfloat[]{%
      \begin{tikzpicture}
          \def\unit{0.7}
            \node (t) at (0, 0) [label=left:T,point];
            \node (f) at (1*\unit, 0) [label=below:F,point];
            \node (r) at (2*\unit, 0) [label=right:R,point];
            \node (h) at (0.5*\unit, 0.7*\unit) [label=above:H,point];
            \node (w) at (1.5*\unit, 0.7*\unit) [label=above:W,point];
            
            \path (t) edge (f);
            \path (f) edge (r);
            \path (h) edge (f);
            \path (h) edge (r);
            \path (w) edge (r);
            \path (w) edge (f);
            \path[bidirected] (h) edge[bend right=40] (t);
            \path[bidirected] (h) edge[bend left=30] (w);
            \path[bidirected] (t) edge[bend right=70] (r);
        \end{tikzpicture}
     }
      %  \subfloat[]{%
    %   \begin{tikzpicture}
    %         \def\unit{0.5}
    %         \node (x) at (0,0) [label=left:X,point];
    %         \node (t) at (2*\unit, 0) [label=right:T,point];
    %         \node (s) at (0, \unit) [label=above:S, square];
            
    %         \path (s) edge (x);
    %         \path (x) edge (t);
    %         \path[bidirected] (x) edge[bend left=60] (t);
    %     \end{tikzpicture}
    %  }
\caption{(a) Selection ADMG which requires intervening on $T$. (b) Causal ADMG for Bike Sharing. }
\label{fig:examples}
\vspace{-0.2in}
\end{figure}

\section{CONNECTIONS WITH EXISTING APPROACHES}
We establish connections between graph surgery and existing proactive approaches, showing that graph pruning  (which finds stable conditional relationships) is a special case of surgery and that surgery has an optimal distributionally robust interpretation.

\subsection{Relationship with Graph Pruning}
We show that graph pruning estimators are in fact surgery estimators, so graph surgery does not fail on problems graph pruning can solve.

\begin{lemma}\label{lemma:superset}
Let $T$ be the target variable of prediction and $\mathcal{G}$ be a selection ADMG with selection variables $\mathbf{S}$. If there exists a stable conditioning set $\mathbf{Z}$ such that $P(T|\mathbf{Z}) = P(T|\mathbf{Z},\mathbf{S})$, then Algorithm \ref{alg:surgery} will not fail on input $(\mathcal{G}, ch(\mathbf{S}), T)$.
\end{lemma}
\begin{proof}
We show $\exists \mathbf{W} \subseteq\mathbf{Z}$ s.t. $P(T|\mathbf{Z})=P_\mathbf{M}(T|\mathbf{W})$. See supplement.
\end{proof}

In the proof of Lemma \ref{lemma:superset} we derived that graph pruning is a special case of graph surgery:
\begin{corollary}
Graph pruning estimators are graph surgery estimators since they can be expressed as conditional interventional distributions.
\end{corollary}

\begin{lemma}
There exists a problem for which graph pruning cannot find a non-empty stable conditioning set but for which graph surgery does not fail.
\end{lemma}
\begin{proof}
As one such example, see Figure \ref{fig:diagnosis}(c).
\end{proof}

From the previous two Lemmas the following Corollary is immediate:
\begin{corollary}
There exists a stable graph surgery estimator for a strict superset of the problems for which there exists a stable graph pruning estimator.
\end{corollary}
We have now shown that graph surgery strictly generalizes graph pruning.

\subsection{Surgery As Distributional Robustness}
We now discuss the optimality of the surgery estimator for adversarial transfer problems in the presence of unstable mechanisms.

Suppose selection ADMG $\mathcal{G}$ defines a prediction problem with target variable $T$ and input features $\mathbf{X} = \mathbf{O}\setminus \{T\}$. Further suppose all variables are continuous such that the prediction problem is regression and that we use the $L_2$ loss. Under classical assumptions that training and test distributions are the same, our goal is to learn a function $f(\mathbf{x})$ that minimizes the expected (squared) loss or \emph{risk}: $E_{\mathbf{o}\sim P(\mathbf{O})}[(t- f(\mathbf{x}))^2]$.

In our setting, however, $P(\mathbf{O})$ varies across domains. Recall that a selection diagram defines a family of distributions $\Gamma$ over $\mathbf{O}$ such that for any particular domain (i.e., setting of $\mathbf{S}$) there exists a $Q_\mathbf{s}\in\Gamma$ such that $P(\mathbf{O}|\mathbf{s})$ factorizes according to (\ref{eq:admg}) and members of $\Gamma$ differ in $\prod_{W\in\mathbf{M}} Q_\mathbf{s}(W|pa(W))$. As opposed to the classical setting, now our goal is to learn a predictor that optimizes for loss across the distributions in $\Gamma$. When constructing reliable models for safety-critical applications in which model failures can be dangerous, a natural choice is to minimize the worst-case or minimax risk across the environments. This can be written as the following game in which we seek the optimal $f$ from the set of continuous functions $\mathcal{C}^0$:
\begin{equation}\label{eq:minimax}
    \min_{f\in\mathcal{C}^0} \sup_{Q_\mathbf{s}\in\Gamma} E_{Q_\mathbf{s}}[(t-f(\mathbf{x}))^2].
\end{equation}

We now give two sufficient conditions under which using the surgery estimator is optimal in that $f_s(\mathbf{x})=E[T|\mathbf{x}\setminus\mathbf{m},do(\mathbf{m})]$ achieves (\ref{eq:minimax}). Proofs are in the supplement.

\begin{theorem}\label{thm:subset}
If $\mathcal{G}$ is such that $P_\mathbf{M}(T|\mathbf{X}\setminus\mathbf{M})$ is identified and equal to $P(T|\mathbf{W})$ for some $\mathbf{W}\subseteq \mathbf{X}$, then $f_s$ achieves (\ref{eq:minimax}).
\end{theorem}

\begin{theorem}\label{thm:frontdoor}
If $\mathcal{G}$ is such that $P_\mathbf{M}(T|\mathbf{X}\setminus\mathbf{M})$ is identified and not a function of $\mathbf{M}$, then $f_s$ achieves (\ref{eq:minimax}).
\end{theorem}

In the case of Theorem \ref{thm:subset}, the surgery estimator reduces to using an invariant subset of features to predict and the result follows from the optimality of graph pruning methods for distributions which share the invariant conditional \citep[Theorem 4]{rojas2018invariant}. Theorem \ref{thm:frontdoor}, however, can correspond to certain cases in which $P_\mathbf{M}(T|\mathbf{X}\setminus\mathbf{M})$ is not an observational conditional distribution. One example is the so-called \emph{front-door} graph \citep{pearl2009causality}, in which there exists no stable pruning estimator while the surgery estimator is optimal across all predictors. Further discussion of this case is in the supplement.

We finally discuss the surgery estimator in the context of stable estimators. Stability is a particularly desirable property because, assuming the estimator is identified, stable estimators can be learned using data from any environment in $\Gamma$. Further, we have the following (proof in supplement):
\begin{theorem}\label{thm:opt-stable}
The surgery estimator is optimal amongst the set of directly transportable statistical or causal relations for predicting $T$.
\end{theorem}
We have discussed cases in which the surgery estimator is minimax optimal across the environments in $\Gamma$ and established that the surgery estimator is the optimal stable predictor. In the context of Theorem \ref{thm:opt-stable}, this means that without additional assumptions (e.g., parametric assumptions about forms of causal mechanisms or assumptions about the magnitude of differences across environments) the surgery estimator provides the best method for prediction that can be trained from and used in any environment.

\section{EXPERIMENTS}
We evaluate the graph surgery estimator in proactive transfer settings in which data from the target distribution is unavailable. The goal of our experiments is to demonstrate that the surgery estimator is stable in situations in which existing methods are either not applicable or suboptimal. To this end, we first consider a simulated experiment for which the true selection diagram is known. Then we apply the surgery estimator to real data to demonstrate its practical utility even when the true selection diagram is unknown. We compare against a naive pooled ordinary least squares baseline (OLS) and causal transfer learning (CT), a state-of-the-art pruning approach \citep{rojas2018invariant}.\footnote{\url{https://github.com/mrojascarulla/causal_transfer_learning}} If CT fails to find a non-empty subset, we predict using the pooled source data mean. On real data we also compare against Anchor Regression (AR), a distributionally robust method for bounded magnitude shift interventions \citep{rothenhausler2018anchor} which requires a causal ``anchor'' variable with no parents in the graph. All performance is measured using mean squared error (MSE).

\subsection{Simulated Data}
We simulate data from zero-mean linear Gaussian systems using the DAG in Figure \ref{fig:diagnosis}(a) considering two variations (full details in supplement).\footnote{This DAG contains no anchor so we cannot apply AR.} The first considers the selection problem in Figure \ref{fig:diagnosis}(a) in which $A$ is a mutable variable, defining a family of DGPs which vary in the coefficient of $K$ in the structural equation for $A$. We generate 10 source environments and apply on test environments in which we vary the coefficient on a grid. Recall that in this DAG the empty set is the only stable conditioning set and CT should model $P(T)$. While this is stable, we expect the performance to be worse than that of the surgery estimator: $P_s \propto P_{A}(T,C) = P(C|T,A)P(T)$ which is able to use additional stable factors.

The MSE as we vary the test coefficient of $K$ is shown in Figure \ref{fig:sim-results}a. As expected, the stable models CT and Surgery are able to generalize beyond the training environments (vertical dashed lines), while the unstable OLS loss grows quickly. However, for small deviations from the training environment OLS outperforms the stable methods which shows that there is a tradeoff between stability and performance in and near the training environment. The gap between CT and Surgery is expected since Surgery models an extra stable, informative factor: $P(C|T,A)$.

We repeat this experiment but consider the target shift scenario in which $T$ is the mutable variable, and the DGPs across environments differ in the coefficient of $K$ in the structural equation for $T$. Now there is no stable conditioning set which violates the assumption of CT. Again, CT used the empty conditioning set $P(T)$ but in this case is unstable so the loss grows quickly in Figure \ref{fig:sim-results}b. As before, OLS is unstable but performs best near the source environments. The surgery estimator $P_s \propto P_{TA}(C) = P(C|T,A)$ is stable and the loss appears constant compared to the unstable alternatives. These experiments demonstrate that stability is an important property when differences in mechanisms can be arbitrarily large. In the supplement we aggregate results for many repetitions of the simulations.

\begin{figure}[!t]
\vspace{-0.2in}
\begin{center}
  \subfloat[]{
    \includegraphics[width=0.22\textwidth]{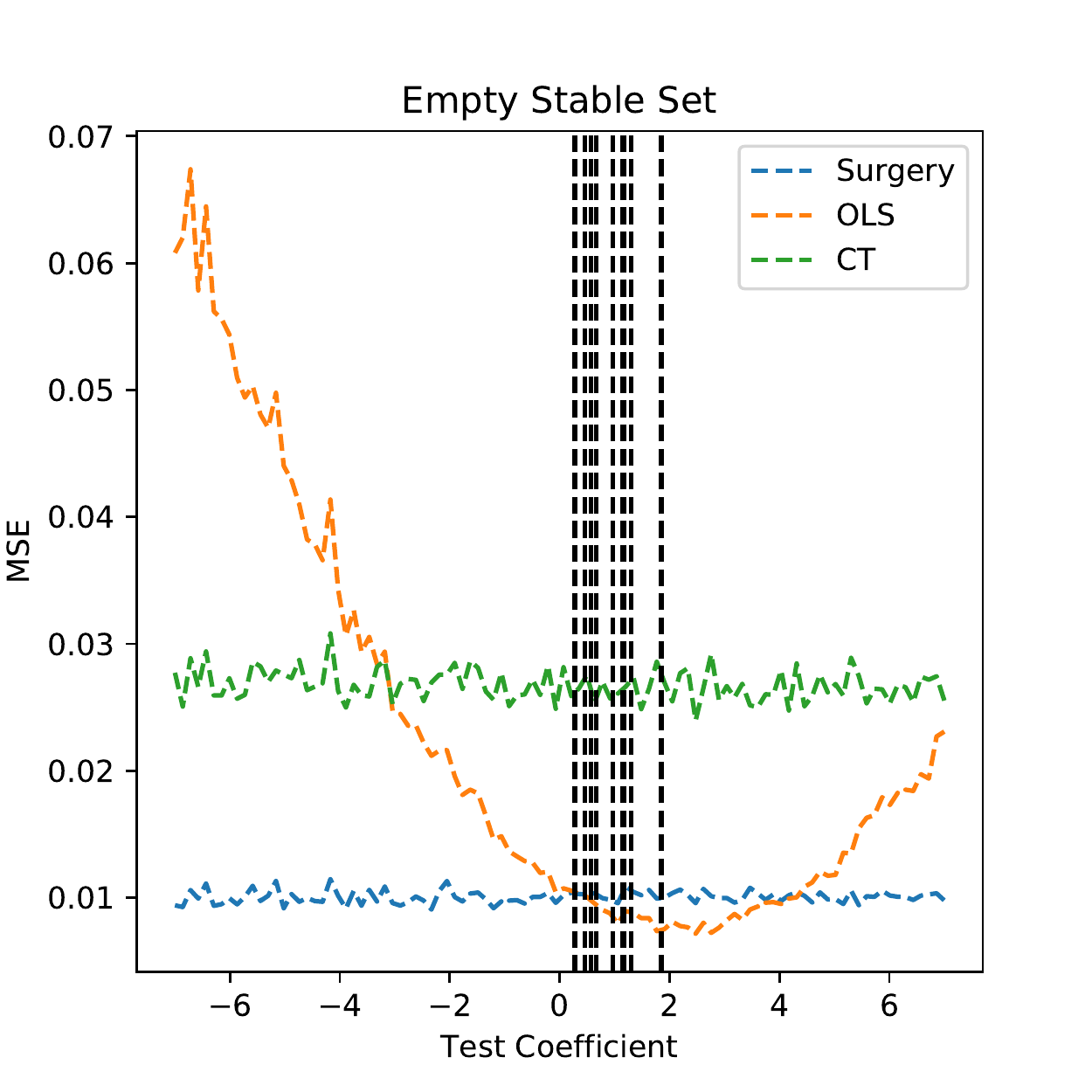}
    }
  \subfloat[]{
    \includegraphics[width=0.22\textwidth]{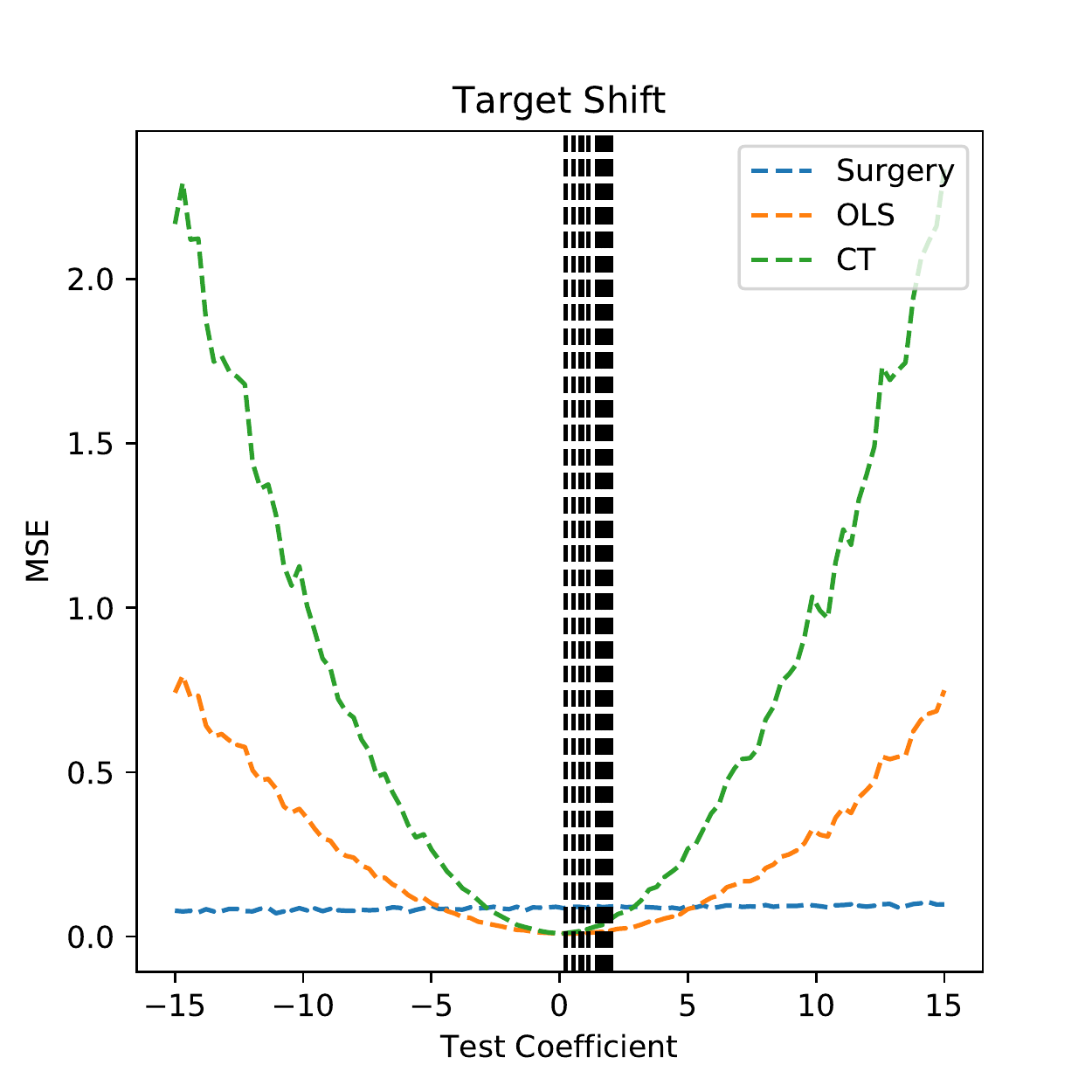}
    }
\end{center}
\caption{(a) MSE in test environments for the Fig \ref{fig:diagnosis}a scenario. (b) MSE in test environment for target shift scenario. Vertical lines denote training environments.}
\label{fig:sim-results}
\vspace{-0.2in}
\end{figure}

\subsection{Real Data: Bike Rentals}
Following \citet{rothenhausler2018anchor} we use the UCI Bike Sharing dataset \citep{fanaee2013,uci} in which the goal is to predict the number of hourly bike rentals $R$ from weather data including temperature $T$, feeling temperature $F$, wind speed $W$, and humidity $H$. As in \citet{rothenhausler2018anchor}, we transform $R$ from a count to continuous variable using the square root. The data contains 17,379 examples with temporal information such as season and year. We partition the data by season (1-4) and year (1-2) to create environments with different mechanisms. We posit the causal diagram in Figure \ref{fig:examples}(b) with confounding caused by unobserved temporal factors, and hypothesize that differences in season result in unstable mechanisms for weather: the mutable variables are $\mathbf{M}=\{H,T,W\}$. If this diagram is true, then no stable pruning estimator exists, so we expect surgery to outperform CT and OLS if the differences in mechanisms are large. The full conditional interventional distribution $P_{THW}(R|F)=P_{THWF}(R)$ is identified and the surgery estimator is given by $\hat{R}_s = \sum_T E[R|T,H,W,F]P(T|H,W)$. We posit linear Gaussians for each term and compute $\hat{R}_s$ using 10,000 Monte Carlo samples. Since AR and CT require data from multiple source environments, for each year (Y), we select one season as the target environment, using the other three seasons as source environments. Since OLS and Surgery do not make use of the season indicator, we simply pool the data for these methods.

We sample $80\%$ of the training/test data 20 times and report the average MSE in Table \ref{table:bike-results} (intervals are one standard error). The surgery estimator performs competitively, achieving the lowest average MSE in 3 of 8 test cases.  When the OLS MSE is high (seasons 3 and 4 in each year), Surgery tends to outperform it which we attribute to Surgery's stability. We also see that CT tends to perform poorly which lends some credibility to our hypothesized selection diagram which dictates that no stable pruning estimator exists. AR's very good performance is expected, since the shift-perturbation assumption seems reasonable in this problem. However, AR requires tuning of a hyperparameter for the maximum magnitude shift perturbation to protect against which is less preferable than stable estimators such as surgery in safety critical applications when the target environment is unknown and could be very different from the source.

\begin{table}[!t]
\caption{MSE on the Bike Sharing Dataset} \label{table:bike-results}
\centering
\scalebox{0.7}{
\begin{tabular}{ccccc}
\hline
Test Data       & OLS           & AR            & CT            & Surgery       \\ \hline
(Y1) Season 1 & 20.8$\pm$0.10 & \textbf{20.5}$\pm$0.10 & 42.2$\pm$2.04 & 20.7$\pm$0.36 \\ 
Season 2 & \textbf{23.2}$\pm$0.05 & \textbf{23.2}$\pm$0.05 & 29.9$\pm$0.09 & 23.8$\pm$0.09 \\
Season 3 & 32.2$\pm$0.14 & 31.4$\pm$0.13 & 32.2$\pm$0.14 & \textbf{29.9}$\pm$0.26 \\
Season 4 & 29.2$\pm$0.08 & 29.1$\pm$0.08 & 29.1$\pm$0.08 & \textbf{28.2}$\pm$0.07 \\ \hline
(Y2) Season 1 & 32.5$\pm$0.11 & \textbf{32.2}$\pm$0.11 & 32.6$\pm$0.15 & 36.1$\pm$0.37 \\
Season 2 & 39.3$\pm$0.11 & \textbf{39.2}$\pm$0.11 & 46.1$\pm$0.12 & 39.5$\pm$0.13 \\
Season 3 & 47.7$\pm$0.17 & \textbf{46.7}$\pm$0.16 & 48.2$\pm$0.22 & 54.8$\pm$0.73 \\
Season 4 & 46.2$\pm$0.16 & 46.0$\pm$0.16 & 46.1$\pm$0.16 & \textbf{44.4}$\pm$0.16 \\ \hline
\end{tabular}
}
\end{table}

\section{CONCLUSION}
Since the very act of deployment can result in shifts that bias a system in practice (e.g., \citet{lum2016predict}), machine learning practitioners need to become increasingly aware of how deployment and training environments can differ. To this end, we have introduced a framework for failure-proofing against and reasoning about threats to reliability due to shifts in environment.\footnote{For an overview of reliability in machine learning and other types of threats see \citet{saria2019tutorial}.} As a means for failure prevention, we have used selection diagrams to identify and express desired invariances to changes in the DGP, and the Surgery estimator as an approach for learning a model that satisfies the invariance specifications. The surgery estimator algorithm serves as a preprocessing step that dictates which pieces of the DGP to model, and subsequently these pieces can be fit using arbitrarily complex models. The surgery estimator has a number of desirable properties, including that it is strictly more applicable than existing graph pruning approaches and it is the optimal stable estimator from a distributionally robust perspective. In future work we wish to consider the tradeoff between stability and performance so as to enable practitioners to make better informed decisions about which invariances to enforce.

% identifying and expressing desired invariances to changes in the DGP, and the Surgery estimator as an approach for learning a model stable to such changes. The surgery estimator finds a stable and identifiable interventional distribution which is expressible as a function of the training data and can be fit using arbitrarily complex models. Further, the interventional distributions are strictly more applicable than the conditional distributions used by existing graph pruning approaches and are optimal from a distributionally robust perspective. In future work we wish to consider methods for when the selection diagram does not entail any identifiable stable predictors. In particular, some form of sensitivity analysis for dealing with uncertainty in the DGP such as infusing bounded-magnitude distributional robustness with prior knowledge of the DGP seems promising.

% \begin{figure}[h]
% \vspace{.3in}
% \centerline{\fbox{This figure intentionally left non-blank}}
% \vspace{.3in}
% \caption{Sample Figure Caption}
% \end{figure}

\subsubsection*{Acknowledgements}
The authors thank Thijs van Ommen for helpful discussions about section 4.2. 

\bibliography{references}

\onecolumn
\newpage
\aistatstitle{Appendix: Learning Predictive Models That Transport}
\appendix
\section{ID Algorithm}\label{appendix:id}
\begin{algorithm}[h]
\SetKwFunction{FIdentify}{Identify}
\SetKwProg{Fn}{Function}{:}{}
 \SetKwInOut{Input}{input}\SetKwInOut{Output}{output}
 \Input{ADMG $\mathcal{G}$, disjoint variable sets $\mathbf{X},\mathbf{Y}\subset \mathbf{O}$}
 \Output{Expression for $P_\mathbf{X}(\mathbf{Y})$ if identified or \texttt{FAIL} if not identified.}
 1. $\mathbf{D}=an_{\mathcal{G}_{\mathbf{O}\setminus\mathbf{X}}}(\mathbf{Y})$\;
 2. Let c-components of $\mathcal{G}_\mathbf{D}$ be $\mathbf{D}_i$, $i=1,\dots,k$\;
 3. $P_\mathbf{X}(\mathbf{Y})=\sum_{\mathbf{D}\setminus \mathbf{Y}}\prod_{i=1}^k \texttt{Identify}(\mathbf{D}_i, \mathbf{O}, P(\mathbf{O}))$\;
 \Fn{\FIdentify{$\mathbf{A}$, $\mathbf{V}$, $Q=Q[\mathbf{V}]$}}{
 \If{$\mathbf{A}==\mathbf{V}$}{\KwRet $Q[\mathbf{V}]$\;}
 \tcc{$C_{\mathcal{G}_\mathbf{V}}(B)$ is c-component of $B$ in $\mathcal{G}_\mathbf{V}$}
 \If{$\exists B \in \mathbf{V}\setminus \mathbf{A}$ such that $C_{\mathcal{G}_\mathbf{V}}(B)\cap ch(B)=\emptyset$}{
 Compute $Q[\mathbf{V}\setminus \{B\}]$ from $Q$\ (Corollary 1)\;
 \KwRet \FIdentify{$\mathbf{A}$, $\mathbf{V}\setminus \{B\}$, $Q[\mathbf{V}\setminus \{B\}]$}\;
 }
 \Else{\KwRet \texttt{FAIL}$(\mathbf{A}, \mathcal{G}_\mathbf{V})$\;}
 }
 \caption{ID($\mathbf{X}$, $\mathbf{Y}$; $\mathcal{G}$)}
 \label{alg:id}
\end{algorithm}

We now restate the identification algorithm (ID) \citep{tian2002general,shpitser2006identification} using the modified presentation in \citet{jabercausal2018}. When the interventional distribution of a set of variables is identified, the ID algorithm returns it in terms of observational distributions (i.e., if the intervention is represented using $do$ notation, then the resulting expression contains no $do$ terms). The ID algorithm is complete \citep{shpitser2006identification}, so if the interventional distribution is not identifiable, then the algorithm throws a failure exception. Note that $\mathcal{G}_\mathbf{V}$ denotes an \emph{induced subgraph} which consists of only the variables in $\mathbf{V}$ and the edges between variables in $\mathbf{V}$.

We will need the following definition:
\begin{definition}[C-component]
In an ADMG, a c-component consists of a maximal subset of observed variables that are connected to each other through bidirected paths. A vertex with no incoming bidirected edges forms its own c-component.
\end{definition}
We also restate the following Corollary \cite[Corollary 1]{jabercausal2018}:
\begin{corollary}
Given an ADMG $\mathcal{G}$ with observed variables $\mathbf{O}$ and unobserved variables $\mathbf{U}$, $V \in \mathbf{X} \subseteq \mathbf{O}$, and $P_{\mathbf{O}\setminus \mathbf{X}}$, if $V$ is not in the same c-component with a child of $V$ in $\mathcal{G}_\mathbf{X}$, then $Q[\mathbf{X}\setminus \{V\}]$ is identifiable and is given by 
\begin{equation*}
    Q[\mathbf{X}\setminus\{V\}]=\frac{P_{\mathbf{O} \setminus \mathbf{X}}}{Q[C(V)]} \sum_V Q[C(V)],
\end{equation*}
where $C(V)$ denotes the c-component of $V$ in the induced subgraph $\mathcal{G}_\mathbf{X}$.
\end{corollary}
This Corollary allows us to derive the post-intervention distribution after intervening on $V$ from the post-intervention distribution after intervening on the variables in $\mathbf{O}\setminus \mathbf{X}$. The modified presentation of Tian's ID algorithm given in \citet{jabercausal2018} is in Algorithm \ref{alg:id}, which computes the identifying functional for the post-interventional distribution of the variables in $\mathbf{Y}$ after intervening on the variables in $\mathbf{X}$ by recursively finding the identifying functional for each c-component in the post-intervention subgraph.

\section{Proofs}
\subsection{Soundness and Completeness of the Surgery Estimator}
\begin{customthm}{1}[Soundness]
When Algorithm 2\ref{alg:surgery} returns an estimator, the estimator is stable.
\end{customthm}
\begin{proof}
Any query Algorithm 2 makes to ID considers intervening on a superset of the mutable variables $\mathbf{X} \supseteq \mathbf{M}$. By Proposition 1 this means the target interventional distribution is stable. From the soundness of the ID algorithm \cite[Theorem 5]{shpitser2006identification}, the resulting functional of observational distributions that Algorithm 2\ref{alg:surgery} returns will be stable.
\end{proof}

\begin{customthm}{2}[Completeness]
If Algorithm 2 fails, then there exists no stable surgery estimator for predicting $T$.
\end{customthm}
\begin{proof}%[Proof Sketch]
Algorithm 2 is an exhaustive search over interventional distributions that intervene on supersets of $\mathbf{M}$ and are functions of $T$. Thus, by completeness of the ID algorithm \cite[Corollary 2]{shpitser2006identification}, if there is a stable surgery estimator, the procedure will find one.
\end{proof}

\subsection{Relationship with Graph Pruning}
\begin{customlemma}{1}
Let $T$ be the target variable of prediction and $\mathcal{G}$ be a selection ADMG with selection variables $\mathbf{S}$. If there exists a stable conditioning set $\mathbf{Z}$ such that $P(T|\mathbf{Z}) = P(T|\mathbf{Z},\mathbf{S})$, then Algorithm 2\ref{alg:surgery} will not fail on input $(\mathcal{G}, ch(\mathbf{S}), T)$.
\end{customlemma}
\begin{proof}
Assume that $P(T|\mathbf{Z})$ is a stable graph pruning estimator. Partition $\mathbf{Z}$ into $\mathbf{X}$ and $\mathbf{W}$ such that $\mathbf{X}\subseteq \mathbf{M}$ and $\mathbf{W}\intersect \mathbf{M} = \emptyset$, and let $\mathbf{V} = \mathbf{M}\setminus \mathbf{X}$. It must be that $T \ci \mathbf{X}|\mathbf{W}$ in $\mathcal{G}_{\underline{\mathbf{X}}}$. If this were not the case then there would be some $X \in \mathbf{X}$ such that there was a backdoor path from $T$ to $X$, and since $X \in ch(\mathbf{S})$ there is a path $T \dots \rightarrow X \leftarrow S$. Because $X$ is conditioned upon, this collider path would be active and $S \not \ci T$, implying $P(T|\mathbf{Z})$ is not stable (a contradiction). Now by Rule 2 of do-calculus, $P(T|\mathbf{X},\mathbf{W}) = P_{\mathbf{X}}(T|\mathbf{W})$. Next consider the remaining mutable variables $\mathbf{V}$. Letting $\mathbf{V}(\mathbf{W})$ denote the subset of $\mathbf{V}$ nodes that are not ancestors of any $\mathbf{W}$ nodes in $\mathcal{G}_{\overline{\mathbf{X}}}$, we will show that $T \ci \mathbf{V}|\mathbf{X},\mathbf{W}$ in $\mathcal{G}_{\overline{\mathbf{X}},\overline{\mathbf{V}(\mathbf{W})}}$. First consider $V \in \mathbf{V}(\mathbf{W})$. For the independence to not hold, there must be an active forward path from $V$ to $T$. But because $V \in ch(\mathbf{S})$, the path $S \rightarrow V \rightarrow \dots T$ is active since $V$ is not conditioned upon, implying contradictorily that $P(T|\mathbf{Z})$ was not stable. Now consider $V \in \mathbf{V} \setminus \mathbf{V}(\mathbf{W})$. For the independence to not hold, either there is an active forward path from $V$ to $T$, or there is an active backdoor path from $V$ to $T$. We previously showed the first case. In the second case, because $V$ is an ancestor of some $W\in \mathbf{W}$ that is conditioned upon, the collider path $S \rightarrow V \leftarrow \dots T$ is active, so $P(T|\mathbf{Z})$ is not stable (contradiction). Thus, by Rule 3 of do-calculus, we have that $P_{\mathbf{X}}(T|\mathbf{W}) = P_{\mathbf{M}}(T|\mathbf{W})$. This is one of the conditional interventional queries that Algorithm 2 considers, so Algorithm 2 will not fail.
\end{proof}

\subsection{Optimality}
\begin{customthm}{3}
If $\mathcal{G}$ is such that $P_\mathbf{M}(T|\mathbf{X}\setminus\mathbf{M})$ is identified and equal to $P(T|\mathbf{W})$ for some $\mathbf{W}\subseteq \mathbf{X}$, then $f_s(\mathbf{x})=E[T|\mathbf{x}\setminus\mathbf{m},do(\mathbf{m})]$ achieves (2):
\begin{equation*}
    f_s\in\argmin_{f\in\mathcal{C}^0} \sup_{Q_\mathbf{s}\in\Gamma} E_{Q_\mathbf{s}}[(t-f(\mathbf{x}))^2].
\end{equation*}
\end{customthm}
\begin{proof}
The structure of this proof follows that of Theorem 4 in \citet{rojas2018invariant} which proves the optimality of using invariant conditional distributions to predict in an adversarial setting.

Consider a function $f\in \mathcal{C}^0$, possibly different from $f_s$. Now for each distribution $\mathbb{Q}\in\Gamma$ corresponding to an environment, we will construct a distribution $\mathbb{P}\in\Gamma$ such that
\begin{equation*}
    \int(t-f(\mathbf{x}))^2 d\mathbb{P} \geq \int (t-f_s(\mathbf{x}))^2 d\mathbb{Q}.
\end{equation*}
Denote the density of $\mathbb{Q}$ by $q(\mathbf{x}, t)$. Note that we have assumed that all distributions in $\Gamma$ correspond to the same graph $\mathcal{G}$ in which $P_\mathbf{M}(T|\mathbf{X}\setminus\mathbf{M})=P(T|\mathbf{W})$ for some $\mathbf{W}\subseteq \mathbf{X}$. Because $\mathbb{Q}\in\Gamma$, $q$ factorizes according to (1) as a product of conditional densities (even when bidirected edges are present, the observational joint can be factorized as a product of univariate conditionals using the c-component factorization \citep{tian2002studies}). To construct the density $p(\mathbf{x}, t)$ of $\mathbb{P}$ from $q$, for $M\in\mathbf{M}$ replace the $q(M|\cdot)$ terms with the marginal density $q(M)$. This is equivalent to removing the edges into $\mathbf{M}$ so notably $P(\mathbf{O}\setminus\mathbf{M}|\mathbf{M})=P_\mathbf{M}(\mathbf{O}\setminus\mathbf{M})$ by rule 2 of $do$-calculus. Thus in $\mathbb{P}$ we have that $P_\mathbf{M}(T|\mathbf{X}\setminus\mathbf{M})=P(T|\mathbf{X})$. But since the full conditional interventional distribution is stable it must be that $P(T|\mathbf{X})=P(T|\mathbf{W})$. So, we know that $q(t|w)=p(t|w)$. Further, we have that $p(\mathbf{w})=q(\mathbf{w})$ since we constructed $\mathbb{P}$ to have the same marginals of $\mathbf{M}$ as $\mathbb{Q}$ and the other terms remain stable across members of $\Gamma$. Thus $q(t,\mathbf{w})=p(t,\mathbf{w})$. Letting $\mathbf{Z}=\mathbf{X}\setminus \mathbf{W}$, we note that $\mathbf{T}\ci \mathbf{Z}|\mathbf{W}$ in $\mathbb{P}$. We now have that

\begin{align*}
    \int(t-f(\mathbf{x}))^2 d\mathbb{P} &= \int_{t,\mathbf{x}} (t-f(\mathbf{x}))^2 p(\mathbf{x},t) d\mathbf{x}dt\\
    &\geq \int_{t,\mathbf{x}} (t-E[T|\mathbf{x}])^2 p(\mathbf{x},t) d\mathbf{x}dt \tag{Conditional mean minimizes MSE}\\
    &= \int_{t,\mathbf{x}} (t-E[T|\mathbf{x}\setminus\mathbf{m},do(\mathbf{m})])^2 p(\mathbf{x},t) d\mathbf{x}dt \tag{Conditional and interventional distributions are equal by construction}\\
    &= \int_{t,\mathbf{w}}\int_{\mathbf{z}} (t-f_s(\mathbf{m},\mathbf{x}\setminus\mathbf{m}))^2 p(\mathbf{z}|\mathbf{w})p(\mathbf{w},t) d\mathbf{z}d\mathbf{w}dt\\
    &= \int_{t,\mathbf{w}}\int_{\mathbf{z}} (t-f_s(\mathbf{w}))^2 p(\mathbf{z}|\mathbf{w})p(\mathbf{w},t) d\mathbf{z}d\mathbf{w}dt\tag{$E[T|\mathbf{x}\setminus\mathbf{m},do(\mathbf{m})]=E[T|\mathbf{w}]$}\\
    &= \int_{t,\mathbf{w}}\int_{\mathbf{z}} (t-f_s(\mathbf{w}))^2 p(\mathbf{z}|\mathbf{w})q(\mathbf{w},t) d\mathbf{z}d\mathbf{w}dt\tag{$q(t,\mathbf{w})$ is stable}\\
    &= \int_{t,\mathbf{w}} \int_{\mathbf{z}} (t-f_s(\mathbf{w}))^2 q(\mathbf{z}|t,\mathbf{w})d\mathbf{z} q(\mathbf{w},t) d\mathbf{w}dt \tag{$(t-f_s(\mathbf{w}))^2$ is not a function of $\mathbf{z}$}\\
    &= \int_{\mathbf{z},t,\mathbf{w}} (t-f_s(\mathbf{w}))^2 q(\mathbf{w},t, \mathbf{z})d\mathbf{w}dt d\mathbf{z}\\
    &= \int_{t,\mathbf{x}} (t-f_s(\mathbf{w}))^2 d\mathbb{Q}
\end{align*}
\end{proof}

\begin{customthm}{4}
If $\mathcal{G}$ is such that $P_\mathbf{M}(T|\mathbf{X}\setminus\mathbf{M})$ is identified and not a function of $\mathbf{M}$, then $f_s(\mathbf{x})=E[T|\mathbf{x}\setminus\mathbf{m},do(\mathbf{m})]$ achieves (2):
\begin{equation*}
    f_s\in\argmin_{f\in\mathcal{C}^0} \sup_{Q_\mathbf{s}\in\Gamma} E_{Q_\mathbf{s}}[(t-f(\mathbf{x}))^2].
\end{equation*}
\end{customthm}
\begin{proof}
The structure of this proof closely follows the structure of the previous proof. 

Consider a function $f\in \mathcal{C}^0$, possibly different from $f_s$. Now for each distribution $\mathbb{Q}\in\Gamma$ corresponding to an environment, we will construct a distribution $\mathbb{P}\in\Gamma$ such that
\begin{equation*}
    \int(t-f(\mathbf{x}))^2 d\mathbb{P} \geq \int (t-f_s(\mathbf{x}))^2 d\mathbb{Q}.
\end{equation*}

We shall again construct $\mathbb{P}$ from $\mathbb{Q}$ such that in $\mathbb{P}$ $P(T|\mathbf{X}\setminus\mathbf{M},do(\mathbf{M}))=P(T|\mathbf{X})$. Note that we have assumed that $P(T|\mathbf{X}\setminus\mathbf{M},do(\mathbf{M})))$ is not a function of $\mathbf{M}$. This usually corresponds to a \emph{dormant independence} or \emph{Verma constraint} \citep{shpitser2008dormant} in the graph: it means that $T\ci\mathbf{M}|\mathbf{X}\setminus\mathbf{M}$ in $\mathcal{G}_{\overline{\mathbf{M}}}$ (the graph in which edges into $\mathbf{M}$ have been deleted). Further discussion of this can be found in the next subsection of the supplement.

Let $\mathbf{Z} = \mathbf{X}\setminus\mathbf{M}$. By Proposition 1 we have that $p(t,\mathbf{z}|do(\mathbf{m}))=q(t,\mathbf{z}|do(\mathbf{m})$ where $p$ and $q$ denote the densities of $\mathbb{P}$ and $\mathbb{Q}$, respectively. Note that recovering the joint density $p(t,\mathbf{z},\mathbf{m})$ from $p(t,\mathbf{z}|do(\mathbf{m}))$ requires multiplying by a functional of the observational distribution $\mathbb{P}$ of the form $p'(\mathbf{m}|t,\mathbf{z})$ (that is, a product of \emph{kernels} \citep{richardson2017nested} or conditional-like univariate densities of $\mathbf{m}$) where $p'$ denotes that this is not an observational conditional density.

\begin{align*}
    \int(t-f(\mathbf{x}))^2 d\mathbb{P} &= \int_{t,\mathbf{x}} (t-f(\mathbf{x}))^2 p(\mathbf{x},t) d\mathbf{x}dt\\
    &\geq \int_{t,\mathbf{x}} (t-E[T|\mathbf{x}])^2 p(\mathbf{x},t) d\mathbf{x}dt \tag{Conditional mean minimizes MSE}\\
    &= \int_{t,\mathbf{z},\mathbf{m}} (t-E[T|\mathbf{z},do(\mathbf{m})])^2 p(\mathbf{z},\mathbf{m},t) d\mathbf{z}d\mathbf{m}dt \tag{Conditional and interventional distributions are equal by construction}\\
    &= \int_{t,\mathbf{z},\mathbf{m}} (t-f_s(\mathbf{m}, \mathbf{z}))^2 p(\mathbf{z},\mathbf{m},t) d\mathbf{z}d\mathbf{m}dt\\
    &= \int_{t,\mathbf{z},\mathbf{m}} (t-f_s(\mathbf{m}, \mathbf{z}))^2 p(\mathbf{z},t|do(\mathbf{m}))p'(\mathbf{m}|t,\mathbf{z}) d\mathbf{z}d\mathbf{m}dt\\
    &= \int_{t,\mathbf{z},\mathbf{m}} (t-f_s(\mathbf{m}, \mathbf{z}))^2 q(\mathbf{z},t|do(\mathbf{m}))p'(\mathbf{m}|t,\mathbf{z}) d\mathbf{z}d\mathbf{m}dt \tag{Stability of $q(\mathbf{z},t|do(\mathbf{m}))$ by Prop 1}\\
    &= \int_{t,\mathbf{m},\mathbf{z}} (t-f_s(\mathbf{z}))^2 q(\mathbf{z},t|do(\mathbf{m}))p'(\mathbf{m}|t,\mathbf{z}) d\mathbf{z}d\mathbf{m}dt\tag{$E[T|do(\mathbf{m},\mathbf{z}]$ is not a function of $\mathbf{m}$}\\
    &= \int_{t,\mathbf{m},\mathbf{z}} (t-f_s(\mathbf{z}))^2 q(\mathbf{z},t|do(\mathbf{m}))q'(\mathbf{m}|t,\mathbf{z}) d\mathbf{z}d\mathbf{m}dt\tag{$(t-f_s(\mathbf{z}))^2$ is not a function of $\mathbf{m}$}\\
    &= \int_{t,\mathbf{m},\mathbf{z}} (t-f_s(\mathbf{z}))^2 q(\mathbf{z},t,\mathbf{m})d\mathbf{z}d\mathbf{m}dt\\
    &=\int_{t,\mathbf{x}} (t-f_s(\mathbf{z}))^2 d\mathbb{Q}
\end{align*}

In the above derivation, note that $p(\mathbf{m},t|do(\mathbf{m}))p'(\mathbf{m}|t,\mathbf{z})$ essentially represents a particular grouping of terms whose products equals $p(t,\mathbf{x})$ (e.g., in a DAG without hidden variables both terms would be products of conditionals of the form $p(v|pa(v))$). Since the integrand of the expectation is not a function of $\mathbf{M}$, we have the independence in the intervened graph, and we constructed $\mathbb{P}$ such that there are no backdoor paths from $\mathbf{M}$ to $T$, we can push the associated expectations inwards and replace them with the $q'$ terms that recover the joint density $q$. To see this in the context of a particular graph see the front-door graph section of the supplement.
\end{proof}

\begin{customthm}{5}
The surgery estimator is optimal amongst the set of directly transportable statistical or causal relations for predicting $T$.
\end{customthm}
\begin{proof}
First consider the set of directly transportable statistical relations for predicting $T$. These will be observational distributions (i.e., no $do$ terms) of the form $P(T|\mathbf{Z})$ for $\mathbf{Z}\subseteq\mathbf{X}$. We have already shown that stable conditioning sets correspond to conditional interventional distributions of the form $P(T|\mathbf{W}, do(\mathbf{M})$ in Lemma 1 and Corollary 1. Thus, we only need to consider directly transportable causal relations (stable conditional interventional distributions). However, Algorithm 2 is exactly an exhaustive search over stable conditional interventional distributions that returns the optimal one, thus the surgery estimator is optimal amongst directly transportable statistical and causal relations.
\end{proof}

\subsection{Front-door Graph}
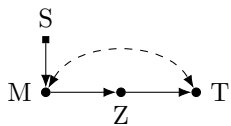
\begin{figure}[!h]
\centering
      \begin{tikzpicture}
          \def\unit{1.0}
          \node (m) at (0, 0) [label=left:M,point];
          \node(z) at (\unit, 0) [label=below:Z, point];
          \node(t) at (2*\unit, 0)[label=right:T, point];
          \node (s) at (0, 0.7*\unit) [label=above:S, square];
          
          \path (m) edge (z);
          \path (z) edge (t);
          \path (s) edge (m);
          \path[bidirected] (m) edge[bend left=60] (t);
        \end{tikzpicture}
\caption{The front-door ADMG.}
\label{fig:front-door}
\end{figure}
Consider the selection ADMG in Fig \ref{fig:front-door}. Notably, for this graph there is no stable graph pruning estimator for predicting $T$. Conditioning on either $\mathbf{Z}$ or $\mathbf{M}$ activates the path $S\rightarrow M \leftrightarrow T$, and conditioning on nothing leaves the path $S\rightarrow M \rightarrow Z \rightarrow T$ active, so there is no stable conditioning set (including the empty set). The full conditional surgery estimator, however, is identified and stable:
\begin{equation*}
    P(T|do(M),Z) = \sum_{m'} P(T|m', Z) P(m').
\end{equation*}
Note that this distribution is not a function of $M$ as it has been marginalized out. This encodes the constraint that $T \ci M | Z$ in $\mathcal{G}_{\overline{M}}$, the graph in which the edges into $M$ are deleted. We see that in the front-door graph, after intervening on $M$ the only relationship between $M$ and $T$ is via the directed chain $M\rightarrow Z\rightarrow T$. Thus $Z$ mediates all of the effect of $M$ on $T$, and the conditional interventional distribution, once computed, is not a function of $M$.

We can use this example to demonstrate how the proof of Theorem 4 works when interventional distributions are different from observational distributions. Now consider a distribution $\mathbb{Q}$ from the family $\Gamma$ that corresponds to this graph. The density factorizes as $q(T,Z,M)=q(T|Z,M)q(Z|M)q(M)$. We will construct a new member of the family $\mathbb{P}$ such that $p(T,Z,M) = p'(T|Z)q(Z|M)q(M)$ where $p'(T|Z)=p'(T|Z,M)=\int_{m'} q(T|Z,m')q(m')dm'$. While the factorization looks different, $\mathbb{P}$ is simply a member of $\Gamma$ that corresponds to the chain without unobserved confounding. Let $f_s(z)=f_s(z,m)=E[T|do(m),z]$ (not a function of $M$). Now consider some function $f(z,m)\in\mathcal{C}^0$:
\begin{align*}
    \int (t-f(z,m))^2 d\mathbb{P} &= \int_{t,z,m} (t-f(z,m))^2 p(t,z,m) dtdzdm\\
    &\geq \int_{t,z,m} (t-E[T|z,m])^2 p'(t|z,m)p(z|m)p(m) dtdzdm\\
    &= \int_{t,z,m} (t-E[T|z,do(m)])^2 p'(t|z)p(z|m)p(m) dtdzdm\\
    &= \int_{t,z,m} (t-f_s(z,m))^2 q(t|z,do(m))q(z|m)q(m) dtdzdm\\
    &= \int_{t,z,m} (t-f_s(z))^2 q(t|z,do(m))q(z|m)q(m) dtdzdm\\
    &= \int_{t,z,m} (t-f_s(z))^2 q(t|z,do(m))q(z)q(m|z) dtdmdz\\
    &= \int_{t,z}(t-f_s(z))^2  q(z) \big(\int_{m'} q(t|z,m')q(m')dm'\big)\big(\int_{m} q(m|z) dm\big) dz dt\\
    &= \int_{t,z}(t-f_s(z))^2  q(z) \big(\int_{m'} q(t|z,m')q(m')dm'\big)\big(\int_{m} q(m|t, z) dm\big) dz dt\\
    &= \int_{t,z, m}(t-f_s(z))^2  q(t, m, z) dm dz dt\\
    &= \int (t-f_s(z))^2 d\mathbb{Q}
\end{align*}

\section{Experiment Details}
\subsection{Hyperparameters for Baselines}
Causal transfer learning (CT) has hyperparameters dictating how much data to use for validation, the significance level, and which hypothesis test to use. In all experiments we set \texttt{valid split} $=0.6$, \texttt{delta}=0.05, and \texttt{use hsic} = \texttt{False} (using HSIC did not improve performance and was much slower).

Anchor regression requires an ``anchor'' variable. In the real data experiment we use season as the anchor. It also has a hyperparameter which dictates the magnitude of perturbation shifts it protects against. We set this to twice the maximum standard deviation of any variable in the training data (including the target).

\subsection{Simulated Experiment}
We generate data from linear Gaussian structural equation models (SEMs) defined by the DAG in Figure 1a:
\begin{align*}
    &K \sim \mathcal{N}(0, \sigma^2)\\
    &T \sim \mathcal{N}(w_1 K, \sigma^2)\\
    &A \sim \mathcal{N}(w_2, \sigma^2)\\
    &C \sim \mathcal{N}(w_3 T + w_4 A, \sigma^2)\\
\end{align*}

We generate the coefficients $w_1,w_2,w_3,w_4 \sim \mathcal{N}(0, 1)$ and take $\sigma^2=0.1^2$.

In simulated experiment 1, $A$ is the mutable variable so across source and target environments we vary the value of $w_2$. Similarly, in experiment 2 (target shift) $T$ is the mutable variable so we vary the value of $w_1$.

We perform both experiments as follows: In each environment we sample 1000 examples. We generate coefficients $w_1,w_2,w_3,w_4 \sim \mathcal{N}(0, 1)$, and take 1000 samples. This is used as the training data for Graph Surgery. Then we generate 1000 samples for each of 9 other randomly generated values of $w_2$ or $w_1$ for experiments 1 and 2, respectively. The 10,000 total samples from 10 environments are used to train the OLS and CT baselines. Then we evaluate on 1000 samples from each of 100 test environments. The $w_2$ (or $w_1$) values are taken from an equally spaced grid. For experiment 1 we consider in $w_2 \in [-100, 100]$ while for experiment 2 we consider $w_1 \in [-10, 10]$. This process is repeated 500 times to yield results on 50,000 test environments.

\begin{figure}[!h]
\begin{center}
\centerline{\includegraphics[scale=0.38]{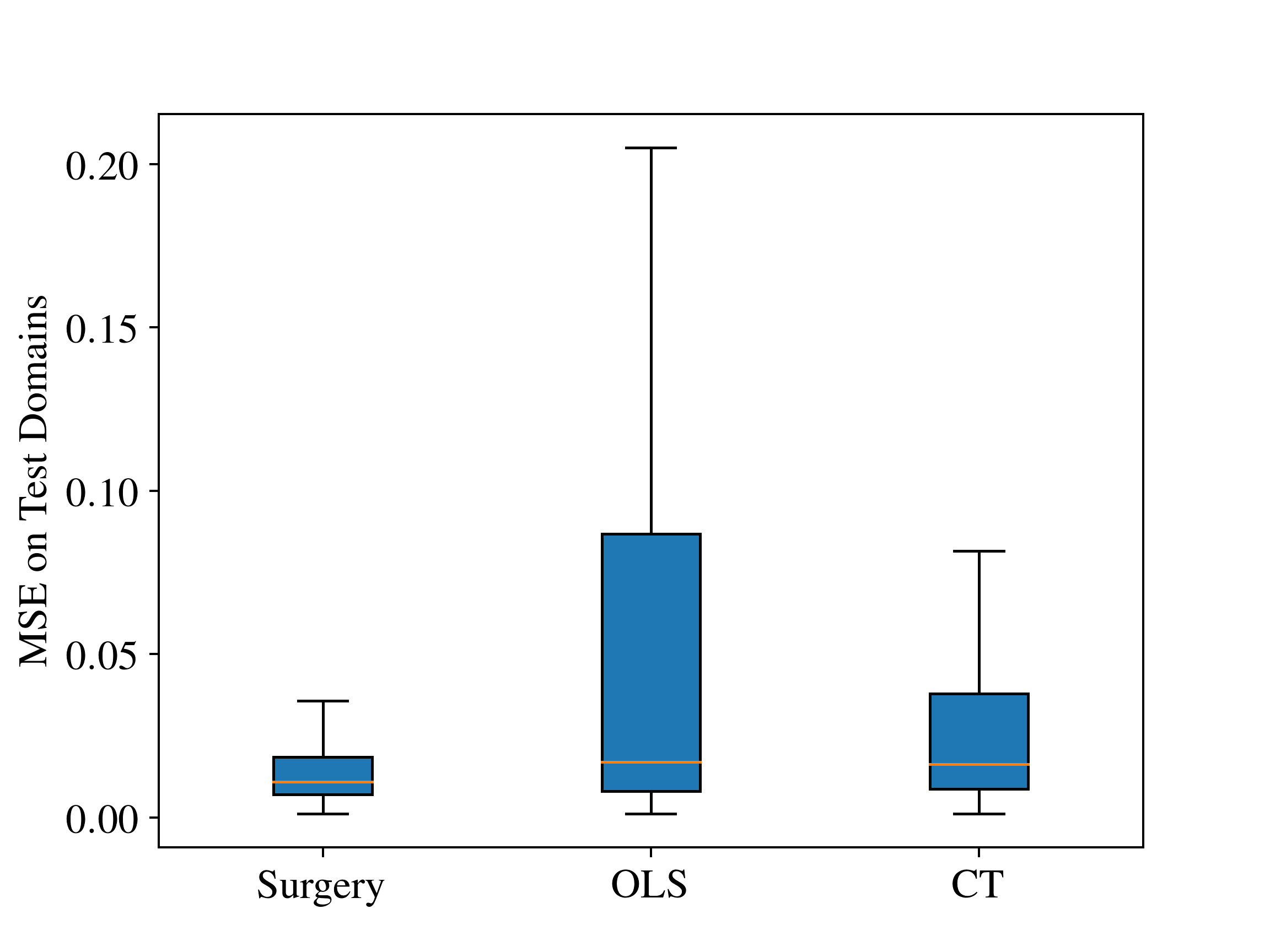}}
\caption{Boxplot of MSE in test environments for the Fig 1a scenario. }
\label{fig:boxplot1}
\end{center}
\end{figure}

The boxplot of the test environment MSEs across the 50,000 test environments for Experiment 1 is shown in Figure \ref{fig:boxplot1}. In this example, Surgery is the only consistently stable model. CT is stable when it selects the empty conditioning set, but in $70\%$ of the 500 runs CT picks all features (i.e., it is equivalent to OLS). We see that the two (at least sometimes) stable methods have much lower variance in performance. Thus, stability implies less variance across environments which is desirable in the proactive transfer setting.

\begin{figure}[!h]
\begin{center}
\centerline{\includegraphics[scale=0.38]{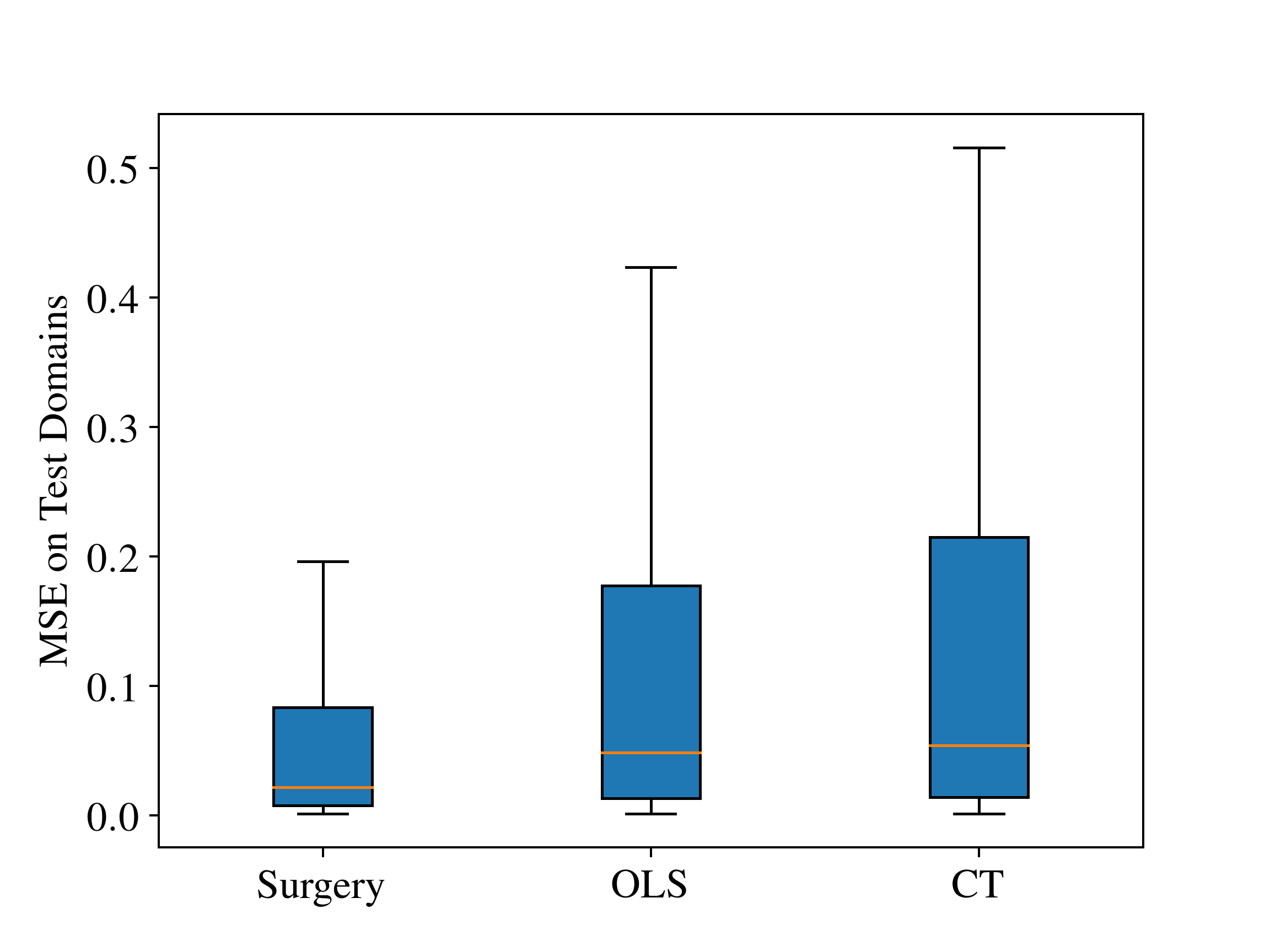}}
\caption{Boxplot of MSE in test environments for the target shift scenario. }
\label{fig:boxplot2}
\end{center}
\end{figure}

The boxplot of the test environment MSEs across the 50,000 test environments for Experiment 2 is shown in Figure \ref{fig:boxplot2}. In this example, Surgery is the only consistently stable model. CT has no stable conditioning set. In $60\%$ of runs CT conditioned on all features. The other times it tended to use the empty set. However, in this experiment $P(T)$ is not stable and uses less information than $P(T|A,C)$ (which OLS models) which is what causes it to have worse performance than OLS. Thus, even in the challenging target shift scenario, graph surgery allows us to estimate a stable model when no stable pruning or conditional model exists.

% \bibliography{references}

\end{document}